\documentclass[nohyperref]{article}

\usepackage{microtype}
\usepackage{graphicx}
\usepackage{subcaption}     
\usepackage{booktabs} 

\usepackage{hyperref}

\usepackage[accepted]{icml2022}

\usepackage{amsmath}
\usepackage{amssymb}
\usepackage{mathtools}
\usepackage{amsthm}

\usepackage[capitalize,noabbrev]{cleveref}

\usepackage{colortbl}
\usepackage{longtable}
\usepackage{tikz}
\usetikzlibrary{positioning}
\usetikzlibrary{shapes, shapes.misc, backgrounds}
\usetikzlibrary{calc}

\input{macros_math.tex}
\input{macros.tex}

\usepackage[textsize=tiny]{todonotes}

\icmltitlerunning{Unimodal Mono-Partite Matching in a Bandit Setting}

\begin{document}

\commentrg[inline]{[DONE] couple ou pair ? => couple}
\commentrg[inline]{[DONE] pas d'espace avant ':', '?', '!'}
\commentrg[inline]{[DONE] arm ou recommendation ? => matching}
\commentmr[inline]{[DONE] $\tilde\ellv$ : leader}
\commentmr[inline]{[DONE] $\tilde\av$ : ordered matching}
\commentmr[inline]{[DONE] $\av$ : matching (sous entendu unordered)}
\commentmr[inline]{[DONE] Remarque : objets notés $i,j$ ou $e_i, e_j$ ? -> $i,j$}
\commentmr[inline]{[DONE] Remarque : \textit{optimum leader} vs \textit{optimum matching} -> ok si pas trop de mix}
\commentmr[inline]{[DONE] Rename $\mu_i$ with $\theta_i$}
\commentmr[inline]{TODO: Introduction, Theorical Analysis, Conclusion}
\commentmr[inline]{DONE: State of the Art, Settings, Methodologie, Algorithm v1, Experimental Results}


\twocolumn[
    \icmltitle{Unimodal Mono-Partite Matching in a Bandit Setting}



    \icmlsetsymbol{equal}{*}

    \begin{icmlauthorlist}
        \icmlauthor{Matthieu Rodet}{ens}
        \icmlauthor{Romaric Gaudel}{ecc}
    \end{icmlauthorlist}

    \icmlaffiliation{ens}{ENS Rennes, F-35000 Rennes, France}
    \icmlaffiliation{ecc}{Univ Rennes, Ensai, CNRS, CREST - UMR 9194, F-35000 Rennes, France}
    \icmlcorrespondingauthor{Romaric Gaudel}{romaric.gaudel@ensai.fr}

    \icmlkeywords{Online Learning to Match, Unimodal Bandit}

    \vskip 0.3in
]



\printAffiliationsAndNotice{\icmlEqualContribution} 

\begin{abstract}

    We tackle a new emerging problem, which is finding an optimal monopartite matching in a weighted graph.
    The semi-bandit version, where a full matching is sampled at each iteration, has been addressed by \cite{ADMA}, creating an algorithm with an expected regret matching $O(\frac{L\log(L)}{\Delta}\log(T))$ with $2L$ players, $T$ iterations and a minimum reward gap $\Delta$.
    We reduce this bound in two steps. First, as in \cite{GRAB} and \cite{UniRank} we use the unimodality property of the expected reward on the appropriate graph to design an algorithm with a regret in $O(L\frac{1}{\Delta}\log(T))$.
    Secondly, we show that by moving the focus towards the main question `\emph{Is user $i$ better than user $j$?}' this regret becomes $O(L\frac{\Delta}{\Tilde{\Delta}^2}\log(T))$, where $\Tilde{\Delta} > \Delta$ derives from a better way of comparing users.
    Some experimental results finally show these theoretical results are corroborated in practice.

\end{abstract}


\section{Introduction}

In some online games, when all servers are running games, players wait in a queue to find a freed server to start a game. As finding a game could take some time, a player may not be available when her game is found, leading to its cancellation and causing extra wait to her opponent. Thus, game developers need to find a solution to cancel as less games as possible. This problem can be formulate as a monopartite matching problem: given a set of players which are each defined by their unknown inner probability of accepting the match, developers have to propose a matching among players maximizing the number of effectively played games.

This question boils down to the longstanding problem of finding matchings in a weighted graph, i.e. a subset of edges without common vertices \cite{lovasz2009matching}. Indeed, the graph could be constituted of the players as vertices and an edge could exists between two players only if they can compete against each other, weighted by the probability of a match occurring between them. In addition to this concrete application, this problem knows various applications in operations research \cite{wheaton1990vacancy}, economics \cite{roth2004kidney} and Machine Learning \cite{mehta2012online}. We consider the online version, where the algorithm is in charge of submitting a sequence of matchings (one per iteration) and gets in return the evaluation of these submissions. This problem is part of the combinatorial semi-bandits ones \cite{cesa2012combinatorial}. 


As in \cite{Komiyama2017, GRAB, ADMA}, in this paper we focus on the setting where the matrix carrying the probability of match between two players is a rank-1 matrix. We use this setting as a toy example to first showcase the power of unimodality when facing a combinatorial bandit setting related to a ranking problem, and to show that we can go beyond current state of the art by fully embracing the main question `\emph{Is user $i$ better than user $j$?}'.

Concretely, first we handle that setting by applying the strategy we introduced in \cite{GRAB} and \cite{UniRank}: (i) we define a graph on the (ordered) matchings such that the expected reward is unimodal on this graph, and (ii) we apply a variant of OUSB \citep{Combes2014} on this graph. This again demonstrate the interest of the unimodality point of view to pull out the intrinsic complexity of a combinatorial semi-bandit setting. Indeed, by denoting $2L$ the number of players, $T$ the number of iterations, and $\Delta$ the minimum reward gap, the naive application of the combinatorial bandit literature leads to an algorithm with a $O(\frac{L^2\log^2(L)}{\Delta}\log(T))$ regret\footnote{Using by example ESCB \cite{combes2015combinatorial}.}, while \citep{ADMA} gets an almost linear regret by accounting for the strong link between rank-1 matching and ranking. In this paper, we show that the unimodality enables an exact linear regret by focusing on the minimal set of comparisons (to the price of an additional $\log\log T$ term).

Secondly, the direct application of GRAB \cite{GRAB} compares two players by comparing the expected regret of two matchings. While this strategy seems obvious, we show in this paper that it is suboptimal. Indeed, we exhibit another criterion with a gap $\tilde\Delta$ greater than the gap $\Delta$ between expected rewards, which drastically reduces the exploration budget of our algorithm and therefore of its expected regret.

This paper is divided in six sections. We start with a formal description of the addressed problem in Section~\ref{seq:settings}. Then, we introduce in Section~\ref{seq:related work} the related work attached to this problem, its extensions and its problem realm. Section~\ref{seq:methodo} is about our methodology, concerning the assumptions made and their consequences. After defining the precise context of the problem, we propose two versions of \ouralgo{} in Section~\ref{seq:algo} and we prove that they enjoy a regret linear in $L$ in Section~\ref{seq:theorical analysis}. Section~\ref{seq:xp results} finally presents numerical results against a state-of-the-art algorithm.

\section{Setting}
\label{seq:settings}

We consider the problem of delivering monopartite matchings in a weighted graph, and give our settings. For any integer $n$, let $[n]$ be the set $\{1, \dots, n\}$. An instance of our problem is a couple $(L, (\rho_{i,j})_{i,j \in [2L]^2})$ where $2L$ is the number of elements to be matched, and for all elements $i,j \in [2L]^2$, $\rho_{i,j}$ is the probability of success associated to the elements $i,j$. A success could be, if $i$ and $j$ are players, that a game occurs, meaning that both player $i$ and player $j$ chose to play the game. Moreover, we assume that there exists a vector $\theta \in \RR^{2L}$ such that $\rho_{i,j} = \theta_i\theta_j$, meaning that $\rho$ if of rank 1. For any element $i$, $\theta_i$ represent the success rate asssociated to player $i$ alone. Finally, without loss of generality, we add the hypothesis that $\theta_1 \geq \dots \geq \theta_{2L}$ to simplify the notations in the rest of the paper.

A matching algorithm is only aware of $L$, and has to deliver $T$ consecutive matchings. A matching is a partition $\av$ of $[2L]$ made of couples, and we note $\av(t)$ the matching returned by the algorithm at each iteration $t$. We note that couples are not ordered, which mean that $\forall i, j \in [2L], i \neq j, \{i, j\} = \{j, i\}$. We denote $\Ac$ the set of all partitions of $[2L]$ made of couples, which correspond to the set of arms of the bandit setting. Formally, $\Ac$ is contains every $\av$ such that:
$$
    \begin{cases}
        \displaystyle \bigcup_{\{i,j\} \in \av} \{i,j\} = [2L] \\
        \forall c \in a, |c| = 2                               \\
        |a| = L
    \end{cases}
$$

We name \textit{optimum matching} the matching $\av^*\defeq \{\{2k-1, 2k\} : k \in [L]\}$. 

At iteration $t$, after recommending the matching $\av(t)$, the algorithm receives a semi-bandit feedback $\cv(t)$ which contains the result of a Bernoulli associated to each couple $\{i,j\} \in \av(t)$, noted $\cv_{i, j}(t)$, of parameter $\rho_{i,j}$. We address the stochastic case, meaning that the way of generating $\cv(t)$ is fixed over iterations and will not change. We focus on the \emph{expected reward}
$$\EE\left[\sum_{(i,j)\in\av(t)} \cv_{i,j}(t)\right] = \sum_{\{i,j\} \in \av(t)} \rho_{i,j}.$$
Due to the rank-1 property, $\av^*$ maximizes the expected reward, as expressed by Lemma~\ref{lem:max-exp-reward}

\begin{lemma}[Maximum Expected Reward]\label{lem:max-exp-reward}
    $$
        \sum_{\{i, j\} \in \av^*} \rho_{i,j} = \max_{\av \in \mathcal{M}} \sum_{\{i,j\} \in \av} \rho_{i,j}
    $$
\end{lemma}

We address the standard bandit objective, we minimize the cumulative expected regret which value is the sum over iterations of the difference between the maximum expected reward and the expected reward of the proposed matching:
$$
    R(T) = \sum_{t=1}^{T} \sum_{\{i, j\} \in \av^*} \rho_{i,j} - \sum_{\{i, j\} \in \av(t)} \rho_{i,j}.
$$

\section{Related Work}
\label{seq:related work}



The  problem addressed in this paper is a combinatorial bandit problem with semi-bandit feedback \cite{cesa2012combinatorial}. Combinatorial bandits have recently known several improvements \cite{combes2015combinatorial,cuvelier2021statistically,degenne2016combinatorial,perrault2020statistical,wang2018thompson}, but these improvements do not account for a strong property of our setting. 
When applied to our, the state of the art combinatorial bandit algorithm, ESCB \cite{combes2015combinatorial}, suffers a regret of $\OO(\frac{L^2\log^2(L) \log(T)}{\Delta})$, $\Delta$ being the expected reward gap between an optimal matching and the best sub-optimal matching. A recent contribution to both monopatite and bipartite cases \cite{ADMA} has, among other things, lowered the monopartite regret bound to a scale of $\OO(\frac{L \log(L)}{\Delta}\log(T))$, through the algorithm \texttt{simple adaptative matching} (SAM). The principle of this algorithm is, knowing $L$ and $T$, to initially group all players together and, at each iteration, divide the current groups in two if the matching algorithm is certain with high probability that every players in one group are better than the other group players. The criterion used to split a group in two parts requires the knowledge of both the total number of players $2L$ and the horizon $T$.

    The bipartite case has also been recently addressed \cite{GRAB}, proposing the algorithm \emph{parametric Graph for unimodal RAnking Bandit} (GRAB) and lowering the expected regret bound to $\OO(\frac{L}{\Delta}\log(T))$. The principle of GRAB is based on a graph w.r.t. which the expected reward is unimodal. At each iteration, GRAB elects a node of this graph as leader and choose from this leader the proposed recommendation. By acquiring knowledge from semi-bandit feedback, GRAB is able to search through nodes and to converge to the optimal recommendation. We show in this paper that we may adapt GRAB to the monopartite case and therefore get a lower bound smaller than the one proposed by SAM.
    Moreover, we show that by using a new criteria to select the played matching, we increase the value of $\Delta$ from $(\theta_i - \theta_{i'})(\theta_j - \theta_{j'})$ for SAM and the original GRAB to $\theta_i(\theta_{i'} - \theta_{j'})$ and $\theta_j(\theta_{i'} - \theta_{j'})$, where $i,i',j,j'$ are players.

    \subsection{Unimodality}\label{sec:SOTA_unimodality}

    Our approach builds upon the unimodality of the expected reward, so let us briefly remind corresponding main results. 

    The \textit{unimodality} has been defined in \cite{Cope2009,Yu2011} and refined in \cite{Combes2014}. Let $G = (V,E)$ be an undirected graph whose vertices are the arms of the bandit, noted $V = \{1,\dots,K\}$, and edges caries a partial order on expected rewards. We assume that there exists a unique arm $k^* \in V$ called optimum such that its expected reward $\mu_{k^*}$ is maximal. Finally, $\forall k \in V, k\neq k^*$ there exists a path $p = (k_1 = k, k_2, \dots, k_{m-1}, k_m = k^*)$ of length $m \in [m]$ depending of $k$, with $(k_i)_{i \in [m]} \in V^{[m]}$, such that $\forall i \in [m-1], (k_i,k_{i+1}) \in E$ and $\mu_i < \mu_{i+1}$. The bandit algorithm is aware of $G$ but ignores the partial order induced by the edges. By relying on $G$, the algorithm is able of browsing efficiently among the arms and to converge to the optimum.

    Both GRAB \cite{GRAB} and OSUB \cite{Combes2014} are designed to benefit from the unimodality. At each iteration, they recommend either the best arm sofar (a.ka. \emph{leader}) or an arm in its neighborhood. The restriction of the exploration to this neighborhood induces a regret which scales as $\OO(\frac{\gamma}{\Delta}\log T)$, with $\gamma$ the maximum degree of $G$, instead of $\OO(\frac{|V|}{\Delta}\log T)$ with independent arms.

    \section{Methodology}
    \label{seq:methodo}

    In this section we add an assumption to the setting and clarify its consequences. While this assumption is not required by \ouralgo{} to get a logarithmic regret, our theoretical analysis needs it to get the linear-in-$L$ regret. Note that \cite{ADMA} assumes the same property in its analysis. Whether this assumption may be discarded or not remains an opened question.

    Our matching algorithm is based on an undirected graph provided with useful properties. We propose in the following a definition of this graph and its properties.

    Each vertex is a matching provided with an order over its couples. We call them ordered matchings, note them $\tilde\av$, and we are now able to refer to the $k$-th couple as $\tilde\av_k$.

    An edge exists between two nodes $\tilde\av$ and $\tilde\av'$ if, and only if, the first can be obtained by switching one element of two successive couples. We note this property $\tilde\av' = \neigpermut{\tilde\av}{i}{j}$, $i,j$ being the two elements switched, and we formally define it as\\
$\exists k \in [L-1], \text{ and let } \{i,i'\} = \tilde\av_k, \{j,j'\} = \tilde\av_{k+1},$
    $$
        \forall n \in [L],
        \begin{cases}
            \tilde\av'_n = \{j,i'\}    & \text{, if } n = k   \\
            \tilde\av'_n = \{i,j'\}    & \text{, if } n = k+1 \\
            \tilde\av'_n = \tilde\av_n & \text{, otherwise}
        \end{cases}
    $$

    We also add a method to switch from ordered matching to matching, that we call $\NtR$, defined for any given ordered matching $\tilde\av$ as
    $$
        \NtR(\tilde\av) \stackrel{def}{=} \{\tilde\av_k : k \in [L]\},
    $$
    the matching constituted of the elements of $\tilde\av$.

    We introduce to the problem setting a new assumption of \textit{inter-pair strict order}, used later in the proof of Lemma~\ref{lem:opt-uni} and stating that
    \begin{assumption}[inter-pair strict order]\label{hyp:strict}
        $\forall k,k' \in [L]^2,\\k<k' \implies \min\left(\theta_{2k-1}, \theta_{2k}\right) > \max\left(\theta_{2k'-1}, \theta_{2k'}\right)$
    \end{assumption}

    We define the property $\pi$ over an ordered matching $\tilde\av$ as
    $$
        \forall k \in [L-1], \{i,j\}=\av_k, \{i',j'\}=\av_{k+1}, \rho_{i,j} \geq \rho_{i',j'}.
    $$
    This property, noted $\pi(\tilde\av)$, ensures that an order over $\rho$ is respected within $\tilde\av$.

    From this property, we construct the optimum ordered matching, called \textit{optimum leader} and noted $\tilde\av^*$, as
    \begin{align}
         &  \NtR(\tilde\av^*) = \av^* 
         &&  \pi(\tilde\av^*) \label{prop:opt}
    \end{align}

    \begin{lemma}[\textit{optimum leader} uniqueness]\label{lem:opt-uni}
        Under Assumption~\ref{hyp:strict}, $\tilde\av^*$ is unique.
    \end{lemma}



    From all the previous definitions and lemmas, we can deduce a relaxed unimodality property over the graph.
    \begin{lemma}[relaxed unimodality]
        Under Assumption~\ref{hyp:strict},
        for any ordered matching $\tilde\av$ satisfying $\pi(\tilde\av)$, if $\tilde\av \neq \tilde\av^*$, we have
        $$
            \exists k \in [L-1], \exists i \in \tilde\av_k, \exists j \in \tilde\av_{k+1}, \mu_{\tilde\av} < \mu_{\neigpermut{\tilde\av}{i}{j}}
        $$
    \end{lemma}

    \begin{proof}
        Proving this lemma is equivalent to proving that for any ordered matching $\tilde\av$, one of the following properties is satisfied:
        \begin{align}
             & \label{prop:neg-pi} \neg \pi(\tilde\av)                                                                                                                      \\
             & \label{prop:optimum} \tilde\av = \tilde\av^*                                                                                                                 \\
             & \label{prop:better-neig} \exists k \in [L-1], \exists i \in \tilde\av_k, \exists j \in \tilde\av_{k+1}, \mu_{\tilde\av} < \mu_{\neigpermut{\tilde\av}{i}{j}}
        \end{align}

        To start with, $\tilde\av$ satisfy one of the following properties, because the last is the conjunction of the negation of the previous ones.
        \begin{align}
             & \label{prop:neg-pi-2} \exists k \in [L-1], \{i,i'\} = \av_k, \{j,j'\} = \av_{k+1}, \rho_{i,i'} < \rho_{j,j'}  \\
             & \label{prop:exch-poss} \exists k \in [L-1], \exists i \in \av_k, \exists j \in \av_{k+1}, \theta_i < \theta_j \\
             & \label{prop:conj-neg-optim} \forall k \in [L-1],  \begin{cases}
                \{i,i'\} = \av_k, \{j,j'\} = \av_{k+1}, \rho_{i,i'} \geq \rho_{j,j'} \\
                \exists i \in \av_k, \exists j \in \av_{k+1}, \theta_i \geq \theta_j
            \end{cases}
        \end{align}
        If $\tilde\av$ satisfies Property (\ref{prop:neg-pi-2}), $\pi$ is not respected. Thus, Property (\ref{prop:neg-pi}) is verified.

        If $\tilde\av$ satisfies Property (\ref{prop:exch-poss}), we can note $\{i,i'\} = \tilde\av_k$ and $\{j,j'\} = \tilde\av_{k+1}$. Without loss of generality, we assume that $\theta_i \geq \theta_{i'}$ and $\theta_j \geq \theta_{j'}$, Property (\ref{prop:exch-poss}) then can be precised to $\theta_{i'} < \theta_j$. Therefor,
        \begin{align*}
            \mu_{\tilde\av} - \mu_{\neigpermut{\tilde\av}{i'}{j}} & = \theta_i\theta_{i'} + \theta_j\theta_{j'} - \theta_i\theta_j - \theta_{i'}\theta_{j'} \\
                                                                  & = (\theta_i - \theta_{j'})(\theta_{i'} - \theta_j)
        \end{align*}
        We deduce $\theta_{i'} - \theta_j < 0$ from Property (\ref{prop:exch-poss}). We then can assume that $\pi(\tilde\av)$, otherwise Property (\ref{prop:neg-pi}) would have been satisfied and the lemma would still be verified. Thus we have $\theta_i\theta_{i'} \geq \theta_j\theta_{j'}$. However, by Property (\ref{prop:exch-poss}), $\theta_i\theta_{i'} < \theta_i\theta_j$, meaning that $\theta_i\theta_j > \theta_j\theta_{j'}$ and finally $\theta_i > \theta_{j'}$. To conclude, $\mu_{\tilde\av} - \mu{\neigpermut{\tilde\av}{i'}{j}} < 0$, fulfilling Property (\ref{prop:better-neig})

        If $\tilde\av$ satisfies Property (\ref{prop:conj-neg-optim}) and assuming $\pi(\tilde\av)$, we have $\forall k \in [L-1],$ let $\{i,i'\} = \tilde\av_k,$ and $\{j,j'\} = \tilde\av_{k+1}, i \geq {i'} \geq j \geq {j'}$. Thus, $\tilde\av$ is, by construction, $\tilde\av^*$, fulfilling Property (\ref{prop:optimum}).

        In every cases, one of Properties (\ref{prop:neg-pi}), (\ref{prop:optimum}) and (\ref{prop:better-neig}) is fulfilled, proving that for any ordered matching $\tilde\av$ satisfying $\pi$, if $\tilde\av \neq \tilde\av^*$,
        $$
            \exists k \in [L-1], \exists i \in \tilde\av_k, \exists j \in \tilde\av_{k+1}, \mu_{\tilde\av} < \mu_{\neigpermut{\tilde\av}{i}{j}}
        $$
    \end{proof}
    This property is widely inspired from the unimodality described in Section~\ref{sec:SOTA_unimodality}. Indeed, when $\tilde\av$ satisfies $\pi$, either $\tilde\av$ is optimal or, if $\tilde\av$ is sub-optimal, there is at least one of its neighbor that has a better regret. We will use this property in the proof of Theorem~\ref{theo:main-v1}.

    \section{\ouralgo{} Algorithm}
    \label{seq:algo}

    Our algorithm, \ouralgo{}, is inspired by the unimodal ranking bandit algorithm GRAB \cite{GRAB}. At each iteration $t$, we select an arm $\av(t)$ in the neighborhood of the leader, the current best arm according to the bandit w.r.t. $\hat\rho$ defined latter, which is learnt online. The main difference is that we don't use the true argmax, but instead a greedy approximation, $\ourargmax$, when electing the leader. The second difference is that the leader is an ordered matching, but not a matching. This is materialized by the notation $\NtR$. \ouralgo{} is described in Algorithm~\ref{alg:GRAB-match-v1}, it has two versions, \ouralgova{} and \ouralgovb{}, and uses the following notations.

    When applied to two elements $i$ and $j$ and at each iteration $t$, we denote
    $$
        \hat\rho_{i, j}(t) \stackrel{def}{=} \frac{1}{T_{i, j}(t)} \sum_{s=1}^{t-1} \mathds{1}\{\{i, j\} \in \av(s)\}\cv_{i,j}(s)\}
    $$
    the average number of games played before iteration $t$ while pairing $i$ with $j$, where $\cv_{i,j}(s)$ equals $1$ if $\{i,j\}$ did lead to a game at iteration $s$ and $0$ otherwise, and
    $$
        T_{i, j}(t) \stackrel{def}{=} \sum_{s=1}^{t-1} \mathds{1}\{\{i, j\} \in \av(s)\}
    $$
    is the number of time $i$ has been paired with $j$ until iteration $t-1$.
    We set $\hat\rho_{i, j}(t)$ to $0$ when $T_{i, j}(t)$ equals $0$.

    \begin{algorithm*}[tb]
        \caption{\ouralgova{} for monopartite matching}\label{alg:GRAB-match-v1}
        \begin{algorithmic}[1]
            \REQUIRE number of items $2L$, optimistic criterion V1 (original) or V2 (new)
            \FOR{$t =  1, 2, \dots$}
            \STATE $\displaystyle\tilde\ellv(t) \gets \ourargmax([2L])$
            \STATE recommend
            $$
                \av(t) =\begin{cases}
                    \NtR(\tilde\ellv(t))
                     & \text{, if } \frac{\Tilde{T}_{\tilde\ellv(t)}(t)}{2L-1} \in \mathds{N},                  \\
                    \displaystyle \argmax_{\substack{\av \in \{\NtR(\tilde\ellv(t))\}\cup\Nc_{Set}(\tilde\ellv(t))}}
                    \sum_{\substack{\{i,i'\}\in\av}} q_{i, i'}(t)
                     & \text{, if } \frac{\Tilde{T}_{\tilde\ellv(t)}(t)}{2L-1} \notin \mathds{N} \text{ and V1} \\
                    \displaystyle \argmax_{\substack{\av = \NtR(\tilde\ellv(t))                                 \\ \text{or } \av \in \Nc_{Set}(\tilde\ellv(t))\text{ with }\av = \NtR\left(\neigpermut{\tilde\ellv(t)}{i}{j}\right),\\ \{i,i'\} \in \NtR(\tilde\ellv(t)) \text{ and } \{j,j'\} \in \NtR(\tilde\ellv(t)) \\}}
                    \max \left(0, q_{i, j'}(t) - q_{i,i'}(t), q_{j, i'}(t) - q_{i,i'}(t)\right)
                     & \text{, if } \frac{\Tilde{T}_{\tilde\ellv(t)}(t)}{2L-1} \notin \mathds{N} \text{ and V2}
                \end{cases} $$
            where $\Nc_{Set}(\tilde\av) = \bigcup\limits_{k=1}^{L-1} \{ \NtR\left(\neigpermut{\tilde\av}{e_1}{e_2}\right) : \ e_1 \in \tilde\av_k,\ e_2 \in \tilde\av_{k+1} \}$
            \STATE observe the games vector $\cv(t)$
            \ENDFOR
        \end{algorithmic}
    \end{algorithm*}

    \begin{algorithm}[tb]
        \caption{$\ourargmax$: a greedy approximation of argmax}\label{alg:ourargmax}
        \begin{algorithmic}[1]
            \REQUIRE Set $\mathcal{E}$ of items to chose
            \ENSURE An approximation of the solution of Equation~\eqref{eq:leader}
            \STATE $\tilde\av \gets $ empty list
            \FOR{$k \in \left[\frac{|\Ec|}{2}\right]$}
            \STATE $\displaystyle e_1, e_2 \gets \argmax_{e_1 \in \mathcal{E}, e_2 \in \mathcal{E} \setminus \{e_1\}} \hat\rho_{e_1, e_2}(t)$
            \STATE $\tilde\av$.append$(\{e_1, e_2\})$
            \STATE $\mathcal{E} \gets \mathcal{E} \setminus \{e_1, e_2\}$
            \ENDFOR
        \end{algorithmic}
    \end{algorithm}

    At each iteration $t$ we denote $\tilde\ellv(t)$ the \textit{leader}, the ordered matching returned by  $\ourargmax$. For any ordered matching $\tilde\av$, we denote
    $$
        \Tilde{T}_{\tilde\av}(t) \stackrel{def}{=} \sum_{s=1}^{t-1} \mathds{1}\{\tilde\ellv(s) = \tilde\av\}
    $$
    the number of time $\tilde\av$ has been leader until iteration $t-1$.

    Finally, we denote the optimistic probability of a couple $\tilde\av_k = \{i, j\}$ by
    $$
        q_{i, j}(t) \stackrel{def}{=} f(\hat\rho_{i,j}(t), T_{i,j}(t), \Tilde{T}_{\tilde\ellv(t)}(t) + 1)
    $$
    where $f(\hat\rho,s,t)$ stands for
    $$
        \sup\{p\in [\hat\rho,1]: s \times \text{kl}(\hat\rho,p) \leq \log(t) + 3 \log(\log(t))\},
    $$
    with
    $$
        \text{kl}(p,q) \stackrel{def}{=} p \log(\frac{p}{q}) + (1-p) \log(\frac{1-p}{1-q})
    $$
    the \textit{Kullback Leibler divergence} from a Bernoulli distribution of mean $p$ to a Bernoulli distribution of mean $q$. By definition, we set $f(\hat\rho,s,t)$ to $\infty$ when $t = 0$, prioritizing exploration.

    Let now explain \ouralgo{}. First, \ouralgo{} uses $\ourargmax$ to elicit the leader. This algorithm is described in Algorithm~\ref{alg:ourargmax} and returns a greedy approximation of
    \begin{equation}\label{eq:leader}
        \argmax_{\tilde\av}
                    \sum_{k=1}^L \hat{\rho}_{\tilde{a}_k}(t).
    \end{equation}
    The couples of the leader are elected one by one. Algorithm $\ourargmax$ choses among the remaining couples the current best possible one w.r.t. $\hat\rho_{i, j}(t)$. Note that by Lemma~\ref{theo:badpi}, the returned leader $\tilde\ellv(t)$ satisfies most of the time $\pi(\tilde\ellv(t))$.

    Finally, \ouralgo{}, given in Algorithm~\ref{alg:GRAB-match-v1}, identifies the leader $\tilde\ellv(t)$ and recommends either the matching associated to $\tilde\ellv(t)$ each $(2L - 1)$-th iterations, or the best matching in the leader's neighborhood in regards to an optimistic criterion. With the original version, the criterion is an optimistic estimate of the number of games to be played, while with the new version it is an optimistic estimate of the increase of the game probability when replacing one of both players in the couple $\{i,i'\}$.

    To conclude the presentation of our algorithm, here we discuss its initialization, its time complexity, and the utility of some choices.

    \begin{remark}[Utility of unordered matching] \label{rq:unordered-matching}
        The neighborhood $\Nc_{Set}(\tilde\av)$ is defined after unordered matchings to divide its size by a factor 2. In our recommendations, couples do not need to be ordered, thus some recommendations in our leader's neighborhood are strictly equivalents w.r.t. the expected reward. For example, for any leader $\tilde\ellv$ any $k \in [L-1]$, by noting $\{i, i'\} = \tilde\ellv_k$ and  $\{j, j'\} = \tilde\ellv_{k+1}$, $\NtR\left(\neigpermut{\tilde\ellv(t)}{i}{j}\right) = \NtR\left(\neigpermut{\tilde\ellv(t)}{i'}{j'}\right)$ because couples are not ordered anymore and the permutation of $i$ with $j$ produces the same couples as the permutation of $i'$ with $j'$, leaving other couples unchanged.
        Therefore, $\{ \NtR\left(\neigpermut{\tilde\ellv(t)}{i}{j}\right) : \ i \in \tilde\av_k,\ j \in \tilde\av_{k+1} \}$ is composed of both matchings $\NtR\left(\neigpermut{\tilde\ellv(t)}{i}{j}\right)$ and $\NtR\left(\neigpermut{\tilde\ellv(t)}{i}{j'}\right)$ while we may expect 4 matchings given the difinition.
        \end{remark}

    \begin{remark}[Initialisation]
        For all $i, j \in [L]^2, i \neq j$, we initialize $q_{i, j}$ to $\infty$ to ensure that every neighbor of an arm which is often the leader is played at least once.
    \end{remark}

    \begin{remark}[Algorithmic Complexity]
        The computation time of \ouralgo{} is polynomial in $L$.

        First, to elect the leader, $\ourargmax$ carry out an iteration over the number of couples $L$. At each iteration, it performs an argmax over at most $L\times L$ values. Thus, the elicitation of the leader is done in $\OO(L^3)$ operations.

        Then, the maximisation when recommending $\av(t)$ is over a set of $2L-2$ recommendations. Indeed, for all $k \in [L-1]$,
        $\{ \NtR\left(\neigpermut{\tilde\ellv(t)}{i}{j}\right) : i \in \tilde\av_k, j \in \tilde\av_{k+1} \}$ is of size 2, according to Remark~\ref{rq:unordered-matching}.
        The union of those disjointed sets of size 2 finally gives us a set of $2L - 2$ matchings.

        Finally for each matching $\av$, the computation of the optimistic criterion cost $\OO(1)$ operations. This is obvious with the new criterion, and it becomes clear with the original criterion as soon as we remark that the maximization of $\sum_{\{i,i'\}\in\av} q_{i, i'}(t)$ is equivalent to the maximization of
        $$
            B_\av(t) = \sum_{\{i,j\}\in \av} q_{i,j}(t) - \sum_{\{i,j\}\in \NtR(\tilde\ellv(t))} q_{i,j}(t)
        $$
        which, by elimination of common terms, reduces to the sum of at most four $q_{i,j}(t)$ values.

        Overall, each recommendation is done in $\OO(L^3)$.
    \end{remark}

    \section{Theorical Analysis}
    \label{seq:theorical analysis}

    \commentrg[inline]{Remarque + lemme à glisser quelque-part :
        \begin{itemize}
            \item rem 1: $\ourargmax_{\tilde\av} \hat{\mu}_{\tilde\av} \neq \argmax_{\tilde\av} \hat{\mu}_{\tilde\av}$ en toute généralité,
            \item mais lemme: $\ourargmax_{\tilde\av} \mu_{\tilde\av} + \epsilon= \argmax_{\tilde\av} \mu_{\tilde\av} + \epsilon$ si $\epsilon$ suffisamment petit
            \item d'où remarque 2: $\ourargmax_{\tilde\av}$ retrouvera le bon une fois qu'on a suffisamment joué
        \end{itemize}
        D'ailleurs ce serait aussi l'occasion de parler du fait que $\ourargmax$ "garantit" $\pi(\tilde{\av})$.
    }


    Before looking at the detailed theoretical analysis, let us give Theorem~\ref{theo:main-v1} and Theorem~\ref{theo:main-v2} which respectively express the regret of \ouralgo{} with the original criterion and the regret \ouralgo{} with the new criterion (denoted \ouralgovb{}).
    
    \begin{theorem}[Upper-bound on the regret of \ouralgo{} assuming inter-pair strict order]\label{theo:main-v1}
        Let $(L, (\rho_{i,j})_{i,j \in [2L]^2})$ be an instance of online matching satisfying Assumption~\ref{hyp:strict}. Then, the expected regret of \ouralgova{} satisfies
        \begin{align*}
                R(T)
                &\leq \sum_{\av \in \Nc_\NtR(\tilde\av^*)} \frac{8}{\Delta_\av}\log T + \OO(\log\log T)\\
                &= \OO(\frac{L}{\Delta} \log T)
        \end{align*}
        where $\Delta_\av \stackrel{def}{=} \mu_{\av^*} - \mu_\av$ and\\
        $\displaystyle\Delta \stackrel{def}{=}
        \min_{\substack{k\in[L-1],\\\{i,i'\}=a_k, \{j,j'\}=a_{k+1}}} (\theta_i-\theta_j)(\theta_{i'} - \theta_{j'})$.
    \end{theorem}
    
    \begin{theorem}[Upper-bound on the regret of \ouralgovb{} assuming inter-pair strict order]\label{theo:main-v2}
        Let $(L, (\rho_{i,j})_{i,j \in [2L]^2})$ be an instance of online matching satisfying Assumption~\ref{hyp:strict}. Then, the expected regret of \ouralgovb{} satisfies
        \begin{align*}
                R(T)
                &\leq \sum_{\substack{k\in[L-1],\\\{i,i'\}=a_k, \{j,j'\}=a_{k+1}}} \frac{8\Delta_{\neigpermut{\tilde\av^*}{i'}{j}}}{\tilde\Delta_{i,i',j'}^2}\log T + \OO(\log\log T)\\
                &= \OO(\frac{L\Delta}{\tilde\Delta^2} \log T)
        \end{align*}
        where $\Delta_\av \stackrel{def}{=} \mu_{\av^*} - \mu_\av$,  $\Delta_{i,i',j} \stackrel{def}{=} 
        \theta_i(\theta_{i'} - \theta_{j'})$\\
        $\displaystyle\Delta \stackrel{def}{=}
        \min_{\substack{k\in[L-1],\\\{i,i'\}=a_k, \{j,j'\}=a_{k+1}}} (\theta_i-\theta_j)(\theta_{i'} - \theta_{j'})$\\
        and $\displaystyle\tilde\Delta \stackrel{def}{=} \min_{\substack{k\in[L-1],\\\{i,i'\}=a_k, \{j,j'\}=a_{k+1}}} \theta_i(\theta_{i'} - \theta_{j'})$.
    \end{theorem}

    Both Theorems are proven following the same path as the one in \cite{GRAB}. The idea is to (1) apply standard bandit analysis to limit the expected regret when the leader $\tilde\ellv(t)$ is the optimum leader $\tilde\av^*$, and to (2) upper-bound the number of iterations $t$ in which $\tilde\ellv(t) \neq \tilde\av^*$ by a $\OO(\log(\log(T))$. To assume $\pi(\tilde\ellv(t))$ in both previous cases, we also show that (3) we can upper-bound the number of iterations such that $\neg \pi(\tilde\ellv(t))$ by a $\OO(1)$.

    Despite the change of the setting, the Theorem and the lemmas used to handle these 3 points are similar (in their expression) to the one in \cite{GRAB}. We give them here for the original criterion and we defer their proof to a complete version of the paper.
    
    Step (1) consists in applying Theorem 2 from \cite{GRAB} which we remind hereafter.

   \begin{theorem}[Upper-Bound on the Regret of KL-CombUCB (Theorem 2 of from \cite{GRAB})]\label{theo:goodleader}
        We consider a combinatorial semi-bandit setting. Let $E$ be a set of elements ans $\mathcal{A} \subseteq {0,1}^E$ be a set of arms, where each arm $\av$ is a subset of $E$. Let's assume that the reward when drawing the arm $\av\in\mathcal{A}$ is $\sum_{e\in\av}c_e$, where for each element $e \in E$, $c_e$ is an independant draw of a Bernoulli distribution of mean $\rho_e \in [0,1]$. Therefore, the expected reward when drawing the arm $\av \in \mathcal{A}$ is $\mu_\av = \sum_{e\in\av} \rho_e$.

        When facing this bandit setting, KL-CombUCB (CombUCB1 equiped with Kullback-Leibler indices) fullfils
        \begin{align*}
            \displaystyle
             & \forall\av\in\mathcal{A}\ s.t.\ \mu_\av \neq \mu^*,                                                                                    \\
             & \EE[\sum^T_{t=1} \mathds{1}\{\av(t)=\av\}] \leq \frac{2K_\av^2}{\Delta_\av^2}\textnormal{log }T + \mathcal{O}(\textnormal{log log }T),
        \end{align*}
        and hence
        \begin{align*}
            R(T)
             & \leq \sum_{\av \in \mathcal{A} : \mu_\av \neq \mu^*} \frac{2K_\av^2}{\Delta_\av}\textnormal{log }T + \mathcal{O}(\textnormal{log log }T) \\
             & = \mathcal{O}(\frac{|\mathcal{A}|K_{max}^2}{\Delta_{min}} \textnormal{log }T)
        \end{align*}

    \end{theorem}
    
    Moreover, the number of iterations where \ouralgo{} elects a sub-optimal leader, is upper-bounded by Lemma~\ref{theo:badleader}.

    \begin{lemma}[Upper-Bound on the number of iterations where \ouralgo{} elect a sub-optimal leader satisfying $\pi$]\label{theo:badleader}
        Under the hypothesis of Theorem~\ref{theo:main-v1} and using its notation,
        $$
            \forall \tilde\ellv \in \Ac, \EE[\sum^T_{t=1} \mathds{1}\{\tilde\ellv(t) = \tilde\ellv\}] = \OO(\log\log T).
        $$
    \end{lemma}

    Finally, the number of iterations where \ouralgo{} elects a leader that does not satisfy $\pi$ is controlled by Lemma~\ref{theo:badpi}.

    \begin{lemma}[Upper-Bound on the number of iterations where \ouralgo{} elect a leader that does not satisfy $\pi$]\label{theo:badpi}
        Under the hypothesis of Theorem~\ref{theo:main-v1} and using its notation,
        $$
            \EE[\sum^T_{t=1} \mathds{1}\{\neg \pi(\tilde\ellv(t))\}] = \OO(1).
        $$
    \end{lemma}

    From these results, we are able to prove Theorem~\ref{theo:main-v1}.

    \begin{proof}[Proof of Theorem~\ref{theo:main-v1}]
        The proof is based on a decomposition of the set of the iterations $[T]$. The goal is to apply the previous results to make the proof easier.
        \begin{align*}
            \displaystyle
            [T] = \bigcup_{\av\in\{\av^*\}\cup\Nc_\NtR(\tilde\av^*)} \{t\in[T]:\tilde\ellv(t) = \tilde\av^*, \av(t)=\av\} \\
            \cup \{t\in[T]:\tilde\ellv(t) \neq \tilde\av^*, \pi(\tilde\ellv(t))\} \cup \{t\in[T]:\neg\pi(\tilde\ellv(t))\}
        \end{align*}

        As for any recommendation $\av$, $\Delta_\av \leq K_\av$, this decomposition gives us the following inequality:
        $$\displaystyle
            R(T) \leq \sum_{\av\in\Nc_\NtR(\tilde\av^*)} \Delta_\av A_\av + KB + KC
        $$
        with
        \begin{equation*}
            \begin{split}
                & A_\av = \EE[\sum^T_{t=1} \mathds{1}\{\tilde\ellv(t) = \tilde\av^*, \av(t)=\av\}],\\
                & B = \EE[\sum^T_{t=1} \mathds{1}\{\tilde\ellv(t) \neq \tilde\av^*, \pi(\tilde\ellv(t))\}],\\
                & C = \EE[\sum^T_{t=1} \mathds{1}\{\neg\pi(\tilde\ellv(t))\}].
            \end{split}
        \end{equation*}

        The terme $A_\av$ is smaller than the number of time the arm $\av$ is chosen by KL-CombUCB playing on the set of arms $\{\tilde\av^*\} \cup \Nc_\NtR(\tilde\av^*)$.
        As any arms differs with $\tilde\av^*$ at at most two positions, Theorem~\ref{theo:goodleader} upper-bounds $A_\av$ by
        $$
            \frac{8}{\Delta^2_\av}\log T + \OO(\log\log T)
        $$
        and thus, as $|\Nc_\NtR(\tilde\av^*)| = 2L-2$, $\sum_{\av\in\Nc_\NtR(\tilde\av^*)} \Delta_\av A_\av = \OO(\frac{L}{\Delta}\log T)$.

        By Lemma~\ref{theo:badleader}, we upper-bound the terme $B$ by
        $$
            B = \OO(\textnormal{log log} T),
        $$
        and by Lemma~\ref{theo:badpi}, we upper-bound the terme $C$ by
        $$
            C = \OO(1),
        $$

        The expected regret of \ouralgo{} is finally the sum of those three terms, which end the proof.
    \end{proof}

    \begin{figure}[t]
        \centering%
        \includegraphics[width=0.7\linewidth]{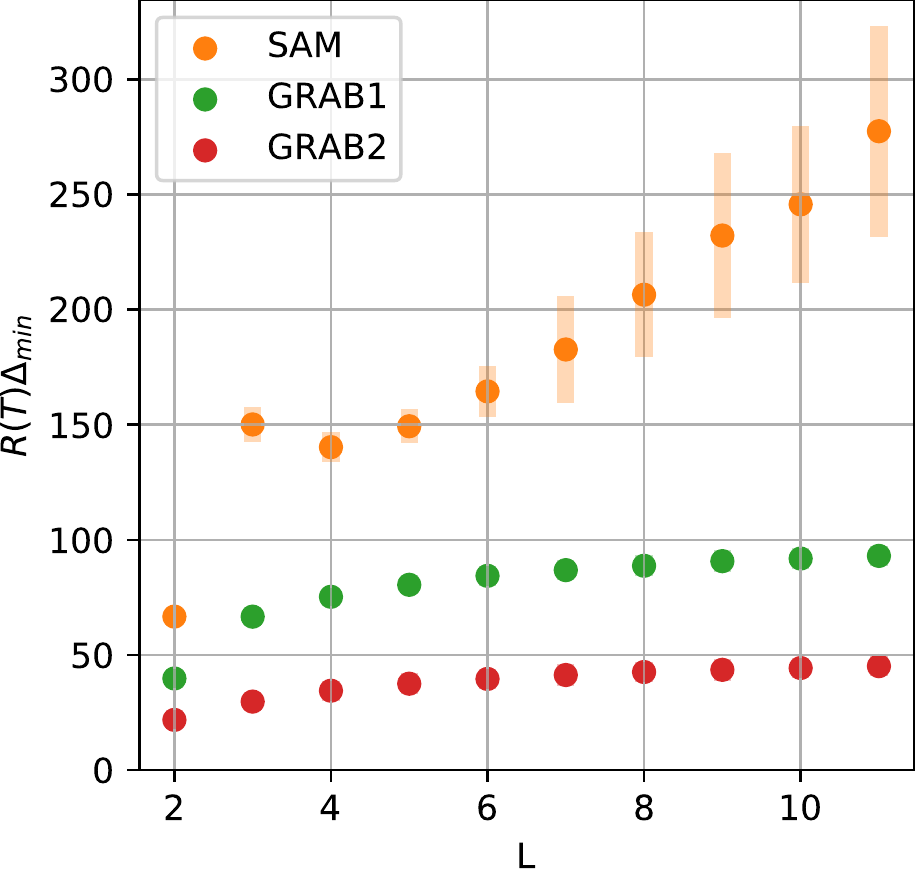}
        \caption{Normalized regret comparison of SAM to \ouralgova{} and \ouralgovb{} versus $L$ in the settings of the experience 1.}
        \label{fig:L}
    \end{figure}

\begin{figure*}[t]
        \centering%
        \begin{subfigure}[b]{0.3\linewidth}
            \centering
            \includegraphics[width=\linewidth]{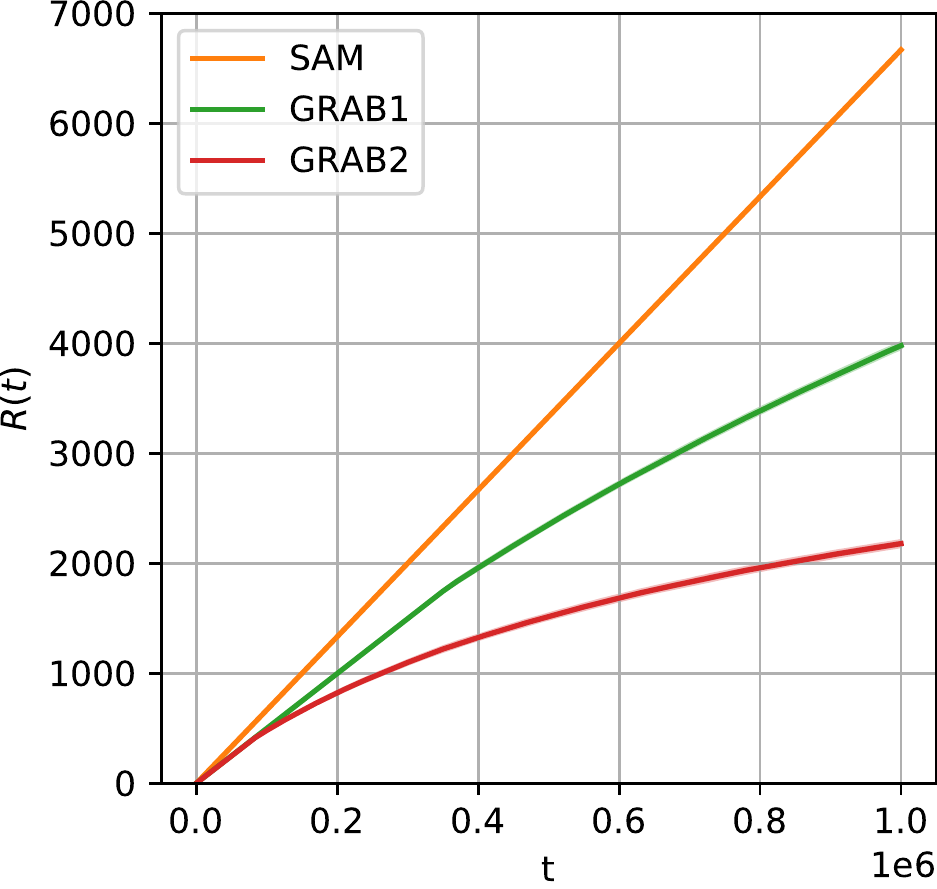}
            \caption{$L=2$}
            \label{fig:T_muNone_L2}
        \end{subfigure}%
        \hfill%
        \begin{subfigure}[b]{0.3\linewidth}
            \centering
            \includegraphics[width=\linewidth]{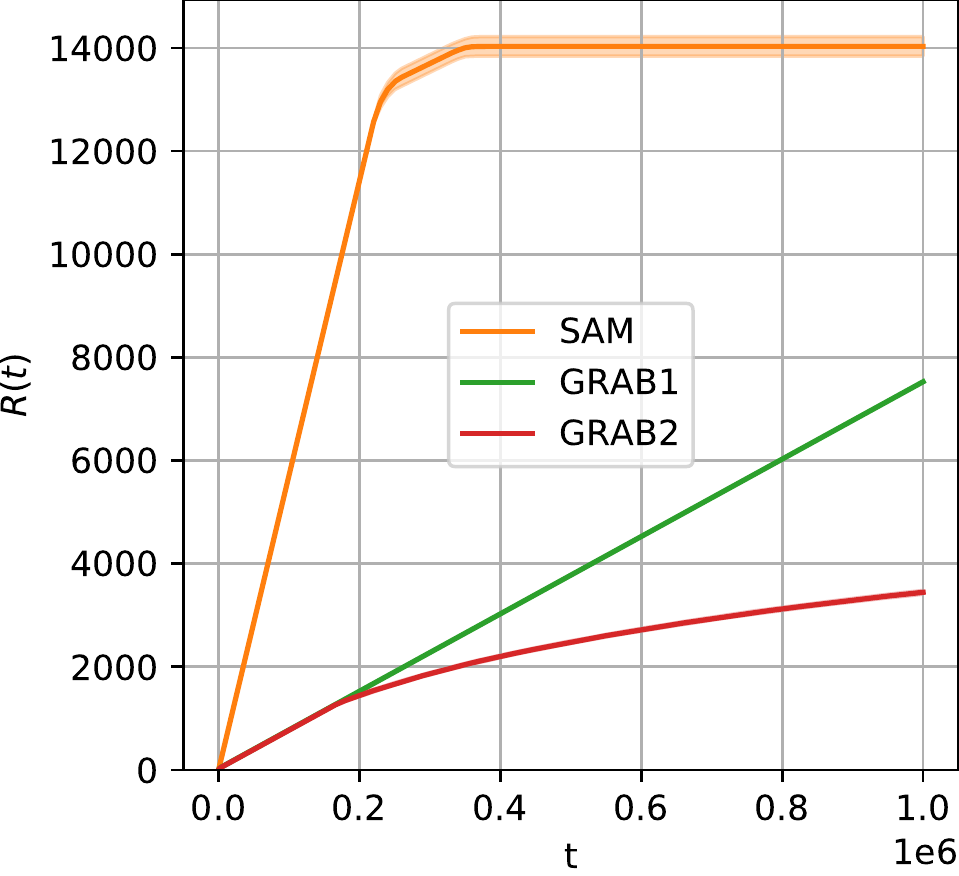}
            \caption{$L=4$}
            \label{fig:T_muNone_L4}
        \end{subfigure}%
        \hfill%
        \begin{subfigure}[b]{0.3\linewidth}
            \centering
            \includegraphics[width=\linewidth]{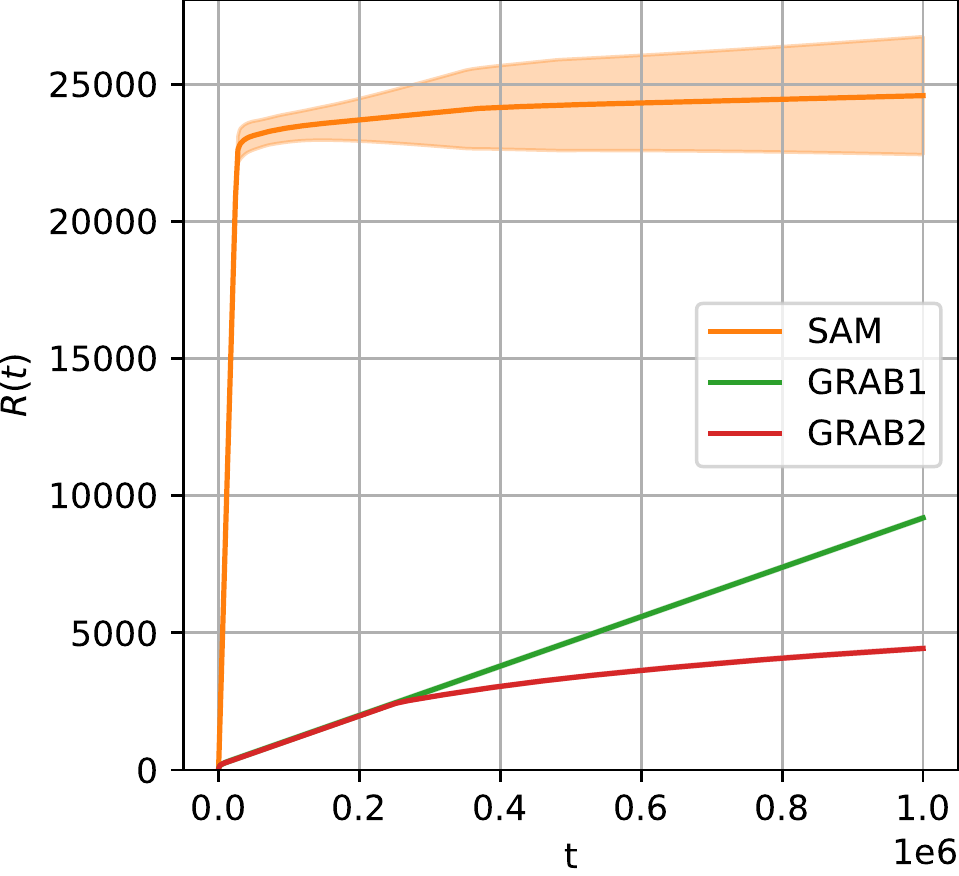}
            \caption{$L=10$}
            \label{fig:T_muNone_L10}
        \end{subfigure}
        \caption{Regret comparison of SAM to \ouralgova{} and \ouralgovb{} versus iteration $t$}
        \label{fig:T_muNone}
    \end{figure*}

    \begin{figure}[t]
        \centering%
        \includegraphics[width=0.7\linewidth]{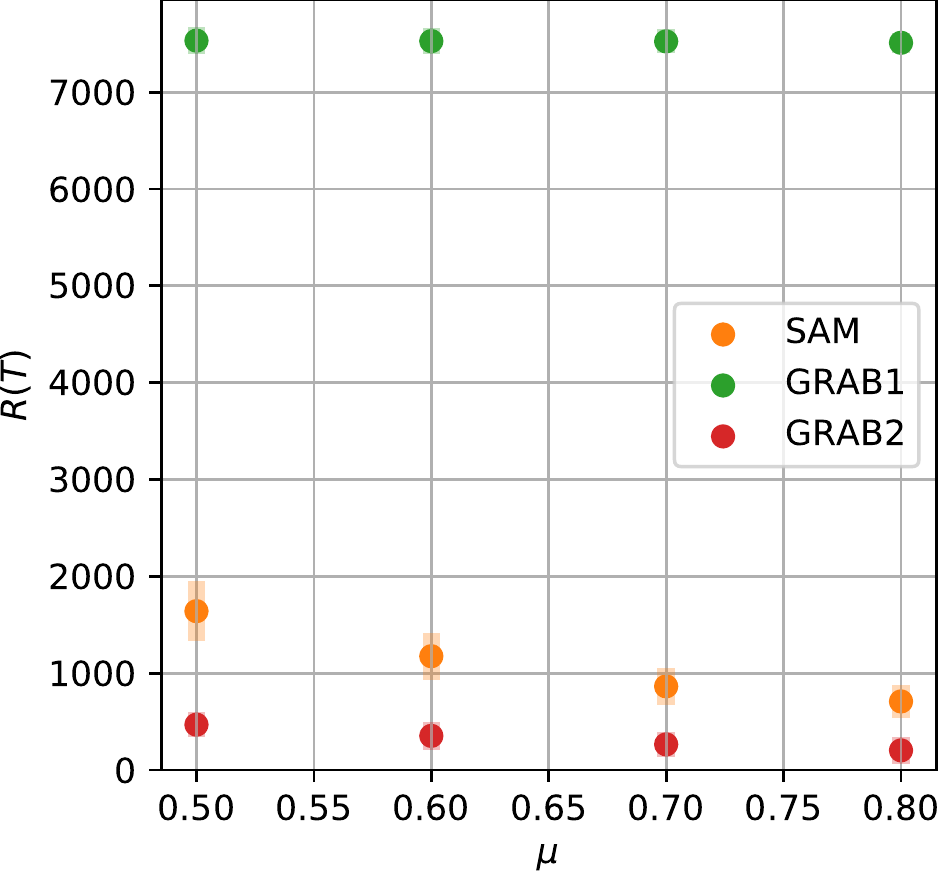}
        \caption{Regret comparison of SAM to \ouralgova{} and \ouralgovb{} versus $\mu$ in the settings of the experience 2.}
        \label{fig:mu}
    \end{figure}

    \begin{figure}[t]
        \centering%
        \includegraphics[width=0.7\linewidth]{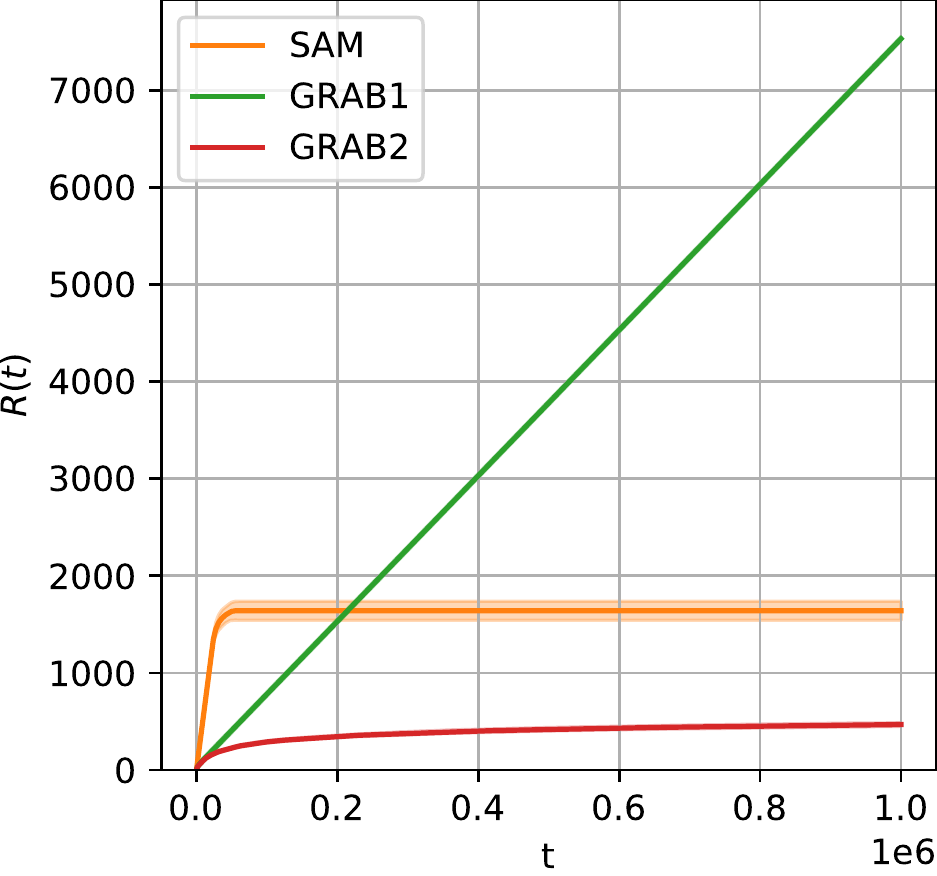}
        \caption{Regret evolution comparison of SAM to \ouralgova{} and \ouralgovb{} among iterations in the settings of the experience 2 and with $\mu = 0.5$.}
        \label{fig:T_mu0.5_L4}
    \end{figure}

    \section{Experimental Results}
    \label{seq:xp results}

    In this section, we compare empirically \ouralgova{} and \ouralgovb{} to SAM \cite{ADMA}. To evaluate the performances, we use simulated data over two experiments with settings similar to the ones used in \cite{ADMA}. Each figure compares the three algorithms when varying the parameters of these settings:
    \begin{itemize}
        \item $L$, the number of couples to recommend (thus, there are $2L$ players);
        \item $T$, the horizon of the simulations;
        \item $\delta$, the gap between each player of consecutive couples in the optimal matching, it can be interpreted as $\delta = \sqrt{\Delta}$;
        \item $\mu$, the mean of all $\theta_i$.
    \end{itemize}

    \begin{remark}[\ouralgo{} simplification]
        For the sake of fairness, our implementation of GRAB use the same optimistic formula than SAM: 
        $$
            q_{i, j}(t) \stackrel{def}{=} \hat\rho_{i, j}(t) + \sqrt{\frac{2\log(t)}{T_{i, j}(t)}},
        $$
        instead of $f(\hat\rho,s,t)$.
    \end{remark}

        In the first experiment, the problems are instantiated with a tuple $(L,\delta)$ such that $0 < (L-1)\delta \leq 1$. We fluctuate $L$ between $2$ and $11$, while $\delta$ and $T$ are respectively fixed to $0.1$ and $10^6$. Finally, we construct a set of $2L$ players whose inner probability is, for each $i \in [L]$, $\theta_{2i-1} = \theta_{2i} = (L-i)\delta$, thus $\mu = \frac{L-1}{2}\delta$. As a result, shown in Figure~\ref{fig:L}, not only we scale, as established, as $\frac{L}{\Delta}$, but we also have a normalized regret divided by a factor $2$ for \ouralgova{} and more than $5$ for \ouralgovb{} w.r.t. the normalized regret of SAM. In Figure~\ref{fig:T_muNone}, we focus on the expected regret evolution along iterations. The difference of strategy between SAM and our algorithms become clear: SAM explores as much as it can at the beginning, and when the algorithm is able to set an order between players with high probability, it divides them in groups, setting a partial order and defining its optimal matching. Meanwhile, our algorithms takes advantage of every information they have to immediately choose better matchings and to limit the comparison between players to as few couples as possible. The early decision leads to a lower regret at first iterations, and the smaller set of couples to evaluate keep the regret low thereafter.

    In the second experiment, we add $\mu$ to the tuple, therefore problems are instantiated by a tuple $(L, \mu, \delta)$ such that $0 \leq \mu - (L-1)\delta$ and $\mu + (L+1)\delta \leq 1$. We then fluctuate $\mu$ between $0.5$ and $0.8$, while $L$, $\delta$ and $T$ are respectively fixed to $4$, $0.1$ and $10^6$. Although it is disadvantageous for \ouralgova{} (as we see in Figure~\ref{fig:L}, others value of $\theta$ would have make it better than SAM), we decided to take this setting for the sake of comparison, since SAM did use this setting in its experimental evaluation. As we can see in Figure~\ref{fig:mu}, results show that \ouralgovb{} has still a lower regret than SAM. Figure~\ref{fig:T_mu0.5_L4} focuses on the case $\mu = 0.5$ and shows the regret evolution over time, highlighting, as Figure~\ref{fig:T_muNone}, the difference of strategy between algorithms. This figure also explains the lack of efficiency of \ouralgova{} by showing a linear regret, which means that it has not the time to learn properly the partial order between arms.

    \section{Conclusion}

    We tackle the problem of finding an optimal matching in a monopartite weighted graph. In  order to solve this problem, we define and use a graph made of all the possible matchings, and use a property of unimodality over the expected reward of each matching. Our algorithms learn online a partial order between each player, by taking advantage of a semi-bandit feedback at each iteration, and give regret upper-bound in $\OO(\frac{L}{\Delta}\log T)$ and $\OO(\frac{L\Delta}{\tilde\Delta^2}\log T)$, reducing at least by a factor $\log L$ the bound w.r.t. the state-of-the-art algorithms. Moreover, unlike SAM, \ouralgo{} does not need $T$ as an initial information.

\bibliography{bib.bib}
\bibliographystyle{icml2022}


\end{document}